\documentclass{article}

 \PassOptionsToPackage{numbers, compress}{natbib}

 \usepackage[preprint]{neurips_2020}

\usepackage[utf8]{inputenc} 
\usepackage[T1]{fontenc}    
\usepackage{hyperref}       
\usepackage{url}            
\usepackage{booktabs}       
\usepackage{amsfonts}       
\usepackage{nicefrac}       
\usepackage{microtype}      

\usepackage{epsfig}
\usepackage{graphicx}
\usepackage{amsmath}
\usepackage{amssymb}
\usepackage{subcaption}
\usepackage{mathptmx}  
\usepackage[ruled,vlined,linesnumbered]{algorithm2e}
\usepackage{amsthm}

\newcommand{\reals}{\ensuremath{\mbox{\bf R}}}
\newcommand{\preals}{\ensuremath{\reals_+}}
\newcommand{\ppreals}{\ensuremath{\reals_{++}}}
\newcommand{\class}[1]{\ensuremath{c_{#1}}}
\newcommand{\classset}{{C}}
\newcommand{\classsetsupport}{{C_{support}}}

\newcommand{\embdim}{\ensuremath{m}}

\newcommand{\gauss}{{N}}
\newcommand{\numclasses}{\ensuremath{{M}}}
\newcommand{\numclassessupport}{\ensuremath{{M_{support}}}}
\newcommand{\numnn}{\ensuremath{{k}}}

\newcommand{\Prob}{\mathrm{Prob}}

\newcommand{\ie}{{\it i.e.}}
\newcommand{\posdefset}[1]{{S}_{++}^{#1}}
\newcommand{\iverson}{I}
\newcommand{\iversonk}[2]{\iverson[{#1}={#2}]}
\newcommand{\ibelongk}[2]{\iverson[{#1} \in {#2}]}
\newcommand{\ordinal}[1]{${#1}$-th}

\newtheorem{theorem}{Theorem}

\newcommand{\prob}{p}
\newcommand{\probest}{\hat{\prob}}

\newcommand{\approptoinn}[2]{\mathrel{\vcenter{
  \offinterlineskip\halign{\hfil$##$\cr
    #1\propto\cr\noalign{\kern2pt}#1\sim\cr\noalign{\kern-2pt}}}}}

\DeclareMathOperator*{\argmax}{arg\,max}

\title{Calibrated neighborhood aware confidence measure for deep metric learning}

\author{%
  Maryna Karpusha \\
  Amazon\\
  \texttt{marynk@amazon.com} \\
  \And
  Sunghee Yun \\
  Amazon\\
  \texttt{sunyun@amazon.com} \\
    \And
  Istvan Fehervari \\
  Amazon\\
  \texttt{istvanfe@amazon.com} \\
}

\begin{document}

\maketitle

\begin{abstract}
Deep metric learning has gained promising improvement in recent years following the success of deep learning. It has been successfully applied to problems in few-shot learning, image retrieval, and open-set classifications. However, measuring the confidence of a deep metric learning model and identifying unreliable predictions is still an open challenge. This paper focuses on defining a calibrated and interpretable confidence metric that closely reflects its classification accuracy. While performing similarity comparison directly in the latent space using the learned distance metric, our approach approximates the distribution of data points for each class using a Gaussian kernel smoothing function. The post-processing calibration algorithm with proposed confidence metric on the held-out validation dataset improves generalization and robustness of  state-of-the-art deep metric learning models while provides an interpretable estimation of the confidence. Extensive tests on four popular benchmark datasets (Caltech-UCSD Birds,  Stanford Online Product, Stanford Car-196, and In-shop Clothes Retrieval) show consistent improvements even at the presence of distribution shifts in test data related to additional noise or adversarial examples.

\end{abstract}
Deep distance metric learning (DDML) aims to learn a deep learning model that maps arbitrary groups of data to a high-dimensional vector embedding space such that the representations of semantically similar items of the same class are closer than the representations of dissimilar items. Such models are applied for tasks such as near-duplicate detection~\cite{b6}, feature-based retrieval~\cite{b1, b2, b3, b4}, clustering~\cite{b5}, visual search~\cite{b7, b8}, etc. The advantage of DDML is that it can be used for challenging extreme multi-label classification problems where the data between classes are unbalanced, scarce, or the number of classes is very large~\cite{b9, b10, b22, extreme_softmax}. In contrast to the traditional classification approach, the goal of DDML is to learn the general concept of similarity between data items opposite to class-specific features. As a result, the trained model can generalize to new classes without retraining.

When the DDML model is trained, the class of a query example is predicted based on the distances to samples with known labels.  A set of examples with known labels used for prediction is known as support or gallery set. However, such approach does not provide a notion of model confidence in the prediction, and the model lacks of easily computable confidence measure that correlates well with the accuracy of the model. Moreover, DDML inherits the disadvantages from its underlying deep neural architecture. Deep neural networks are shown to be sensitive to small input perturbations. For example, natural changes in the data distribution like noise, blurring, and JPEG compression can lead to a significant decrease in the model performance~\cite{b34, b37, b51}. Consequently, one can even engineer such adversarial samples that are imperceptible to human observers, yet completely fool the model~\cite{b23, b25, b26, b27, b28}.

To address these problems, different confidence scores for deep learning models have been proposed~\cite{b27, b41, towards_nn_knows_when_fails, pac, distance-based-confidence, attributes-based-confidence}. However, they are either known to have the significant computational overhead or not applicable in metric learning settings. To overcome these limitations, we propose an approach to approximate the true correctness likelihood that can leverage learned spatial relations of similar and dissimilar items in the embedding space. In particular, we propose a novel confidence score called NED that calculates the Normalized sum of Exponential of the Distances to the nearest neighbors in the embedding space. We provide theoretical evidence derived from the Bayes' theorem that NED approximates the distribution of data points for each class by a Gaussian kernel smoothing function and calculates the conditional probability of a requested point belonging to a specific class. Similar to temperature scaling used to calibrate confidence in classification networks with softmax loss~\cite{b68}, our approach uses normalization to improve probability estimates of the true correctness likelihood. Only a single parameter is fitted to the held-out validation data. Therefore, unlike fitting the original neural network, the algorithm comes with generalization guarantees based on traditional statistical learning theory.

Even though in this paper we focus on the task of visual object classification as it is easier to analyze, our approach can be extended to other modalities of deep metric learning such as speech recognition~\cite{b58} or natural language processing~\cite{b59}.

The main contributions of this paper are the following:
\begin{itemize}
	\item We propose a novel confidence score called NED that enables the computation of reliable and interpretable confidence scores in the deep metric learning setting. We are mathematically motivated to use the proposed NED algorithm as an approximation of the data distribution for each class by a Gaussian kernel smoothing function.
	\item We present state-of-the-art results with our approach on four popular benchmark datasets Caltech-UCSD Birds,  Stanford Online Product, Stanford Car-196, and In-shop Clothes Retrieval using different state-of-the-art DDML models. We show that the robustness of DDML models with the proposed approach can be improved. The proposed NED confidence score improves model generalization to adversarial examples and more natural distortions in the data.
\end{itemize}

\section{Background}

\subsection{Deep Distance Metric Learning}
Suppose  a supervised training on given $N$ independent and identically distributed (i.i.d.) instances
$\{(x_{i}, y_{i})\}_{i=1}^N$ with $x_i \in \reals^n$ and $y_i \in \classset$ from unknown joint distribution of $x$ and $y$
where
$n$ is the dimension of the input space and
$\classset = \{ \class{1}, \ldots, \class{\numclasses}\}$ is the set of the labels in the training set.
The goal of DDML training is to learn a mapping function $f:\reals^n \to \reals^\embdim$
such that distances between sample pairs belonging to the same class in the embedding space
are smaller than
those between sample pairs belonging to the different classes.
To quantify the distance we define a function $d:\reals^\embdim \times \reals^\embdim \to \preals$
representing the distance in the embedding space where $\embdim$  is the dimension of the embedding space.
Then the mapping $z=f(x)$ with deep neural network model should satisfy the following relation:
\begin{equation}
d(z_i, z_j)
\ll
d(z_k, z_l)
\end{equation}
for all $1\leq i,j,k,l \leq N$ such that $y_i=y_j$ and $y_k\neq y_l$. Euclidean distance or cosine distance is typically used for $d$ in DDML models.

The performance of this mapping is measured by some loss function $L:\reals^\embdim \times \reals^\embdim \to \preals$. Early approaches were based on contrastive loss~\cite{contrastiveloss} and triplet loss~\cite{tripletloss}
where the loss function is defined on pairs or triplets of samples in order to minimize intra-class distances and maximize inter-class distances. Since computing the loss on every possible pair or triplet is intractable, recent approaches propose either better sampling strategies~\cite{b4, b16, hierarchical_triplet} or loss functions that consider the relationship of all samples within the training batch. Parametric models have also been proposed with the idea of storing some information about the global context ~\cite{b1, normproxies, ddml_deep_face_recognition}. Several model ensembles have also been explored, focusing on improving classification and retrieval performance using boosting~\cite{BIER} or attending diverse spatial locations~\cite{WonsikAttentionEnsemble}.

\subsection{Confidence in Deep Learning}

In optimization theory, a model is defined robust if it can perform well under a certain level of uncertainty~\cite{b50}.
However, it has been shown that neural networks are susceptible to small intentional or unintentional shifts in the data distribution~\cite{b37, b53}.
In real-world decision-making systems, it is important to indicate whether or not the prediction is reliable.
To achieve this, given a requested point $x\in\reals^n$,
we associate the confidence estimate $\hat{p}(x)$ for every class prediction $\hat{y}(x)$
where $\hat{y}:\reals^n \to \classset$ and $\hat{p}: \reals^n \to [0,1]$.
The confidence estimate $\hat{p}(x)$ is said to be calibrated if it represents the true probability of the correctness $p$~\cite{b68}.

As deep neural networks are shown to provide overconfident predictions, multiple post-processing steps were proposed to produce calibrated confidence measures.

\textbf{Calibration on the held-out validation data.} Different parametric and non-parametric approaches where the logits are used as features to learn a calibration model from a held-out validation data have been proposed~\cite{b68, confidence_bayesian_binning, confidence_histogram_binning, confidence_isotonic_regression}.  Simple a single-parameter variant of Platt
scaling~\cite{platt_scalling}, known as temperature scaling, is often an effective method at obtaining calibrated probabilities~\cite{b68, pac}.
The key insight is that in the temperature scaling approach, only a single parameter is fit to the validation data. Therefore,  unlike fitting the original neural network, the temperature scaling algorithm comes with generalization guarantees based on traditional statistical learning theory.

\textbf{Bayesian approximation.} Bayesian deep neural networks evaluate distributions over the models or their parameters. This approach proposes more accurate and tractable approximation of the uncertainty, but at the cost of expensive computation during training and inference~\cite{b42}. Some recent alternatives have been proposed to approximate predictive uncertainty. For example, Gal and Ghahramani propose considering dropout as a way of ensembling, which approximates Bayesian inference in deep Gaussian processes~\cite{b41}. Unfortunately, empirical results show that one needs to run inference at least 100 times to achieve accurate approximation, which can be infeasible in practice.

\textbf{Support set based uncertainty estimation.}  Papernot et al.~\cite{b20} introduce a trust score that measures the conformance between the classier and $k$-nearest neighbors on the support example. It enhances robustness to adversarial attacks and leads to better calibrated uncertainty estimates~\cite{b7}. Jiang et al.develop this idea further and compute the trust score on deeper layers of DNN than the input to avoid the high-dimensionality of inputs~\cite{trust_or_not}.

The proposed NED algorithm described in the next Section~\ref{sec:ned-section} can be considered as calibration on the support set. The algorithm is designed specifically for DDML models.

\section{NED Algorithm}
\label{sec:ned-section}
Our goal is to find prediction $\hat{y}$ and confidence $\hat{p}$ for requested point $x$ from test data given support set $\{(x_{i}, y_{i})\}_{i=1}^{N_{support}}$ and learned mapping function $z=f(x)$.  The classes between train and support sets can be disjoint if samples share same learned similarity concept.  $\classsetsupport = \{ \class{1}, \ldots, \class{\numclassessupport}\}$ is the set of  labels in support and test sets.

Since metric learning enables a similarity comparison directly in the embedding space using distance metric,
the simplest way to obtain a prediction $\hat{y}$ is to compare $z=f(x)$ to a set of samples from support set in the embedding space and pick the class of the nearest example. Then intuitively, we can assume that the magnitude of this distance can represent the confidence in the prediction.
However, such a distance is an unbounded positive value, and we need to calibrate it to obtain an interpretable probability estimate of the true correctness likelihood. Furthermore, outliers  in the support set can easily lead to misclassification since multiple neighbors are not considered.

To overcome this difficulty, we propose the NED confidence score. To show that it very closely approximates the true correctness likelihood,
we first propose the following theorem.
The proof can be found in Appendix~\ref{sec:ned-interpretation}.

\begin{theorem}
\label{theorem}
Given embedding $f:\reals^n \to \reals^\embdim$ and support set $\{(x_{i}, y_{i})\}_{i=1}^{N_{support}}$
where $x_i\in\reals^n$ and $y_i\in\classsetsupport$
the probability $p$ of $x\in\reals^n$ belonging to $\class{j} \in \classsetsupport$,
\ie, $y=\class{j}$, can be approximate by:
\begin{equation}
\label{eq:theorem}
\hat{p}(y\in \class{j}) \propto\
\frac
{\sum_{i=1}^{N_{support}} \exp\left( - \|z - z_{i}\|_2^2 /T \right) \ibelongk{y_i}{\class{j}}}
{\sum_{i=1}^{N_{support}} \exp\left( - \|z - z_{i}\|_2^2 /T \right)}.
\end{equation}
where $z_i = f(x_i)$ for $i=1,\ldots, N_{support}$ and
$T$ > 0 is a parameter to be tuned.
\end{theorem}
\addtocounter{theorem}{-1}

When $T = 1$, the equation in (\ref{eq:theorem}) is a simple softmax function of the negatives of squares of the Euclidean distances.
However, similar to the bandwidth selection in kernel density estimation~\cite{b64},
we can control the degree of smoothing applied to the samples using $T$.
For example, as $T$ becomes larger, the relative influence of near samples becomes smaller.
On the other hand, if $T$ becomes smaller, the relative influence of near samples becomes larger.
Therefore the optimizing of $T$ is equivalent to finding $T$ which
provides the best relative distances
in terms of probability estimate for the true correctness likelihood.

We optimize $T$ with respect to the negative log likelihood on the support set instead of optimizing on held-out validation dataset. Because the parameter $T$ does
not change the relative magnitudes of the softmax function, we can preserve the spatial order of the nearest neighbors to the requested point $x$ in the embedding space.

Suppose that we chose $\{x_i\}_{i=1}^\numnn$  nearest neighbors with labels $\{y_i\}_{i=1}^\numnn$ for the confidence score calculation.
Because the rest of the data points are far enough from $x$ in the embedding space,
$\exp\left( - \|z - z_{i}\|_2^2/T \right)\approx 0$ for $i>\numnn$,
thus the probability in (\ref{eq:theorem}) becomes
\begin{eqnarray}
\hat{p}(y\in c_{j})
\label{eq:ned}
=
\frac
{\sum_{i=1}^\numnn  \exp\left(- \|z - z_{i}\|_2^2/T\right)\ibelongk{y_i}{\class{j}}}
{\sum_{i=1}^\numnn \exp\left(-\|z - z_{i}\|_2^2/T\right)}.
\label{eq:ned-conf-level}
\end{eqnarray}

The equation in (\ref{eq:ned}) defines the \emph{NED confidence score}.
Below we show the precise NED algorithm description.

\IncMargin{1em}
\begin{algorithm}[H]
	\SetAlgoLined
	\SetKwInOut{Input}{Inputs}\SetKwInOut{Output}{Output}
	\Input{query point $x$, DDML model $z=f(x)$, support set $\{(x_{i}, y_{i})\}_{i=1}^{N_{support}}$, optimal $T$. }
	\Output{prediction $\hat{y}$ with estimated confidence $\hat{p}$.}
	\BlankLine
	$\{z_i\}_{i=1}^{N_{support}}$ = $\{f(x_i)\}_{i=1}^{N_{support}}$   \{Obtain embedding vectors for samples in support set\}\;
	$z=f(x)$ \{Obtain embedding vector for query sample $x$\}\;
	find $k$ nearest neighbors $\{z_i\}_{i=1}^\numnn$ and their labels $\{y_i\}_{i=1}^\numnn$\;
	\For{$j\leftarrow 1$ \KwTo \numclassessupport}{
		$\hat{p}(y=y_{j})=
		\frac
		{\sum_{i=1}^\numnn  \exp\left(- \|z - z_{i}\|_2^2)/T\right)\iversonk{y_i}{y_{j}}}
		{\sum_{i=1}^\numnn \exp\left(-\|z - z_{i}\|_2^2)/T\right)}$\;
	}
	$\hat{p} = \max\limits_{1\leq j\leq \numclassessupport} \hat{p}(y=y_{j}) $\;
	$\hat{y} = y_{j^\ast}$ where $j^\ast = \argmax\limits_{1\leq j\leq \numclassessupport} \hat{p}(y=y_{j}) $\;
	\Return{$\hat{y}$, $\hat{p}$}
	\caption{NED algorithm}
	\label{alg:ned}
\end{algorithm}\DecMargin{1em}

As we show in section ~\ref{sec:evaluation}, one of the advantages of using such a post-processing algorithm is that it can be used to improve generalization and robustness of already trained DDML models. Moreover, it can be combined with known defenses like adversarial training~\cite{b53} or adversarial logit pairing~\cite{b55} to improve further adversarial robustness. Mao \textit{et al.}\ proposes adversarial training specially modified with deep metric learning settings~\cite{b56}. By carefully sampling examples for metric learning, the learned representation increases robustness to adversarial examples and help detect previously unseen adversarial samples.

\subsection{Comparison with $k$-nearest neighbor (kNN) and weighted kNN algorithms}

Non-parametric classifiers like $k$-nearest neighbors (kNN) and weighted $k$-nearest neighbors (WkNN) can also be used to improve the classification performance for already trained DMML models. However, as we show in section~\ref{sec:evaluation}, they provide non-calibrated confidence estimation. We use kNN and WkNN as baselines to compare with the proposed NED algorithm.

The kNN classifier is a simple non-parametric classifier that predicts the label of an input based on a majority vote from labels of the $k$ neighbors in the embedding space. Intuitively, the confidence score for every class can be selected as the percentage of nearest neighbors labels belonging to $c_j$ class:
\begin{equation}
\label{eq:knn_confident_neighbor}
\hat{p}(y \in c_{j}) = \frac{1}{k} \sum_{i=1}^k \ibelongk{y_i}{c_j}
\end{equation}

The robustness of kNN has already been shown from both theoretical perspectives and empirical analyses~\cite{b44, b46, b41, b47}. However,
the main disadvantage of kNN is that its reliability depends critically on the value of $k$.
Therefore many weighted $k$-nearest neighbor (WkNN) approaches have been proposed~\cite{b69, b70, b71} where the closer neighbors are weighted more heavily than the farther ones. We have experimented with various weighted approaches
and achieved the most reliable and calibrated estimation of the confidence score using an approach described in \cite{b70} and \cite{b69}.

In this case, if the distance-weighted function $w_{i}:\reals^n \to \preals$ is defined,  the confidence score can be selected as the weighted percentage of nearest neighbors belonging to $c_j$ class:
\begin{equation}
\label{eq:wknn_confident_neighbor}
\hat{p}(y \in c_{j})
= \frac{
	\sum_{i=1}^k  w_{i} \ibelongk{y_i}{c_j}
}{
	\sum_{i=1}^k  w_{i} }.
\end{equation}
For these methods, the weights $w_i$ are linear functions of the distance between $z$ and $z_i$, \ie, in the embedding space.

To compare the performance of NED algorithm with kNN and WkNN algorithm, we use equations~(\ref{eq:knn_confident_neighbor}) and~(\ref{eq:wknn_confident_neighbor}) to calculate $\hat{p}$ in algorithm~\ref{alg:ned}.

\section{Evaluation}
\label{sec:evaluation}

We consider the following scenarios to evaluate the performance of the proposed NED algorithm.

First, we train the state-of-the-art DDML model with the normalized, temperature-weighted version of the cross-entropy loss following the protocol described in~\cite{b73}. Minimizing the cross-entropy can be seen as an approximate bound-optimization algorithm for minimizing many popular in DDML pairwise losses~\cite{ddml_deep_face_recognition, ddml_multi_similarity, ddml_scalable_nca, tripletloss, hierarchical_triplet, contrastiveloss}. Therefore, we use this approach as a baseline that gives state-of-the-art results with simpler hyperparameter and sampling strategy. Empirical experiments with other popular DDML approaches can be found in Appendix~\ref{sec:results-ddmls}.
When DDML model is trained, we evaluate performance of the model complemented with Algorithm~\ref{alg:ned}.

Second, we empirically evaluate the robustness of our approach in the presence of the distribution shift in test data. In particular, we repeat our experiments when test data contains small common distortions (image transformations related to JPEG compressions, different illumination conditions, camera quality) or when test images are modified using adversarial white-box attacks.

\subsection{Metrics}
For all experiments, we use the following metrics:

\textbf{Accuracy.} We use an accuracy metric to evaluate reliability for challenging extreme multi-class classification problems. Classification accuracy is equivalent to the Recall@1 metric in image retrieval~\cite{b3}.

\textbf{Reliability Diagrams.} A reliability diagram such as the one shown in Figure~\ref{fig:reliability_diagrams}
is a visual representation of confidence metric calibration ~\cite{b72}.
They show the relationship between expected accuracy and confidence estimation.
To estimate the expected accuracy from finite samples, we split test data into $M$ bins. For each bin, the mean predicted confidence score is plotted against the true fraction of positive cases. Both metrics should be near the diagonal line if the model is well calibrated.

\textbf{Expected Calibration Error.} While the reliability diagram is a useful visual representation method of confidence calibration,
it does not show proportion of samples in each bin. Therefore, we use Expected Calibration Error (ECE) to evaluate calibration~\cite{b68} which is defined as:
\begin{equation}
\mathrm{ECE} = \frac{1}{N} \sum_{m=1}^M |B_{m}|  |\text{acc}(B_{m}) - \text{conf}(B_{m})|.
\end{equation}
where $N$ is the number of samples in test dataset,
$M$ is the number of the bins,
$B_m$ is the index set for the $m$th bin,
and $\text{acc}(B_{m})$ and $\text{conf}(B_{m})$
are the accuracy and average confidence for the $m$th bin respectively which are defined as:
\begin{equation}
\mathrm{acc}(B_{m})=\frac{1}{|B_{m}|}\sum_{i \in B_{m}} \ibelongk{\hat{y}_i}{\class{i}}
\mbox{ and }
\mathrm{conf}(B_{m})=\frac{1}{|B_{m}|}\sum_{i \in B_{m}} \hat{p}(y_i \in \class{i}).
\end{equation}

\subsection{Datasets}
We conduct our experiments using DDML models on four benchmark datasets: Caltech-UCSD Birds (CUB-200)~\cite{b11},  Stanford Online Product (SOP)~\cite{b3}, Stanford Car-196 (CARS196)~\cite{b12} and In-shop Clothes Retrieval~\cite{in-shop}.

We follow the common evaluation protocol for these datasets~\cite{b73}. In particular, the object categories between train and test sets are disjoint. This split makes the problem more challenging since deep networks can overfit to the categories in the train set and generalization to unseen object categories could be poor.

\section{Results}\label{ref:results}

Table~\ref{tab:clean_retreival_preformance} shows the performance of DDML model complemented with NED algorithm, and the performance comparison with three baselines kNN, WkNN~\cite{b70}, and WkNN~\cite{b69}.
We provide the accuracy reported in \cite{b73} and the accuracy of our experiment for the case the label of the first nearest neighbor in the embedding space is used for prediction, \ie, 1NN.
The difference in accuracy is caused by using different initialization parameters during training.
The results presented in the table demonstrate that the proposed approach using NED algorithm outperforms all the other approaches
in both accuracy and ECE.
For all experiments, the accuracy for the model with NED algorithm is higher at least by 0.3\% and at most by 7.3\% compared to 1NN. Similarly, using kNN and different versions of WkNN algorithm improves classification accuracy, which demonstrates that outliers to the training distribution can be identified more accurately at test time when more than the first neighbor is considered.

\begin{table*}[h]
	\begin{center}
		\scalebox{0.97}{
		\begin{tabular}{|c|c|c|c|c|c|c|c|c|}
			\hline
			& \multicolumn{2}{|c|}{\textbf{CUB-200}} & \multicolumn{2}{|c|}{\textbf{CARS196}} & \multicolumn{2}{|c|}{\textbf{SOP}}  & \multicolumn{2}{|c|}{\textbf{InShop}} \\
			\hline
			& Accuracy & ECE & Accuracy & ECE  & Accuracy & ECE & Accuracy & ECE  \\
			\hline
			\emph{1NN(reported)} & \emph{67.6} & - & \emph{89.1} & - & \emph{80.8} & - & \emph{90.6} & - \\
			\hline
			1NN & 67.2  & - & 89.1 & - & 81.2 & - & 90.9 & - \\
			\hline
			kNN & 73.8 & 9.2 & 90.1 & 10.0 & 81.2 & 24.4 & 90.9 & 32.5  \\
			\hline
			WkNN \cite{b70} & 74.3 & 20.5 & 91.3 & 12.2 & 81.2 & 5.2 & 91.0 & 17.8  \\
			\hline
			WkNN \cite{b69} & 74.3 & 20.7 & 91.3 & 12.3 & 81.2 & 5.2 & 91.1 & 17.9  \\
			\hline
			NED(proposed) & \textbf{74.9} &\textbf{2.4}  & \textbf{91.5} & \textbf{1.5} & 81.2 & \textbf{2.2} & \textbf{91.3} & \textbf{0.3} \\
			\hline
		\end{tabular}
	}
		\caption{The accuracy and ECE of proposed NED algorithm and the four baselines (1NN, kNN, WkNN~\cite{b70}, and WkNN~\cite{b69}). The best results of each column are shown in \textbf{boldface}. The proposed approach consistently outperforms the baseline methods for all experiments. }
		\label{tab:clean_retreival_preformance}
	\end{center}
\end{table*}

While kNN and WkNN provide more reliable predictions, their confidences do not represent the interpretable confidence of predictions. The ECE is significantly higher than for the proposed NED approach (at least by 3\% and at most by 31.2\%).

The SOP dataset contains very scarce data (from 2 to 12 images per class). While the proposed approach does not improve classification accuracy compare to baselines, it provides better calibrated confidences. The ECE of the proposed approach is 2.2\% whereas that of WkNN is 5.2\%.  This difference demonstrates that NED can be used for uncertainty estimation to detect cases when the model most probably misclassifies for even such scarce dataset like SOP.

Figure~\ref{fig:reliability_diagrams} shows reliability diagrams for CARS dataset.
We see that the kNN, WkNN~\cite{b70}, and WkNN~\cite{b70} tend to be overconfident in its predictions.
On the other hand, NED algorithm produces much better confidence estimation
at the cost of tuning a single parameter $T$.
Also, all the bins are well calibrated by NED algorithm.

\begin{figure}[th]
	\centering
	\begin{subfigure}{0.245\linewidth}
		\includegraphics[width=\linewidth]{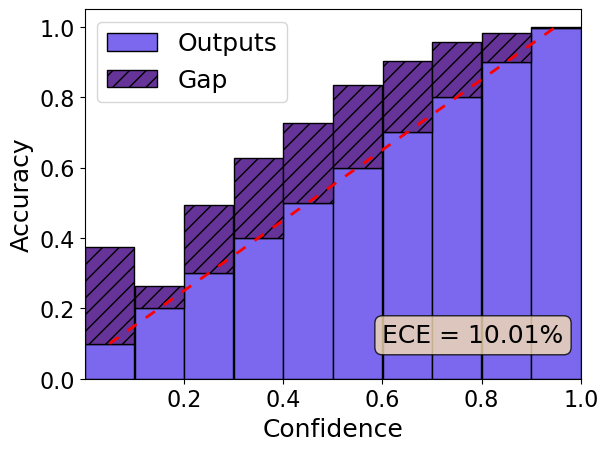}
		\caption{kNN}
	\end{subfigure}%
	\begin{subfigure}{0.245\linewidth}
		\includegraphics[width=\linewidth]{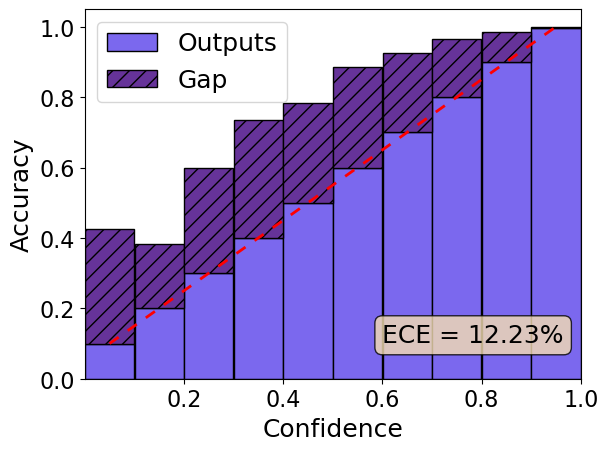}
		\caption{WkNN~\cite{b70}}
	\end{subfigure}
	\begin{subfigure}{0.245\linewidth}
		\includegraphics[width=\linewidth]{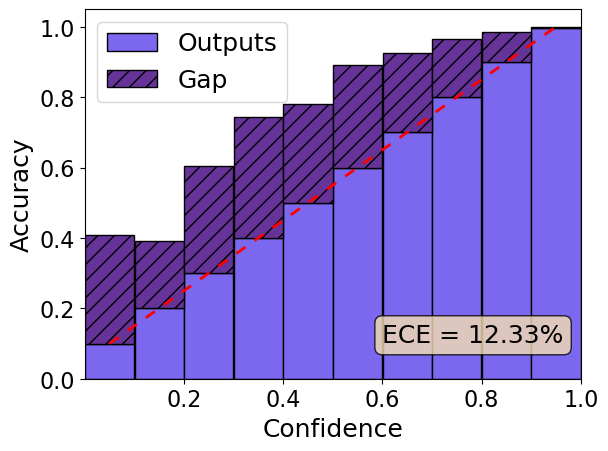}
		\caption{WkNN~\cite{b69}}
    \end{subfigure}
	\begin{subfigure}{0.245\linewidth}
		\includegraphics[width=\linewidth]{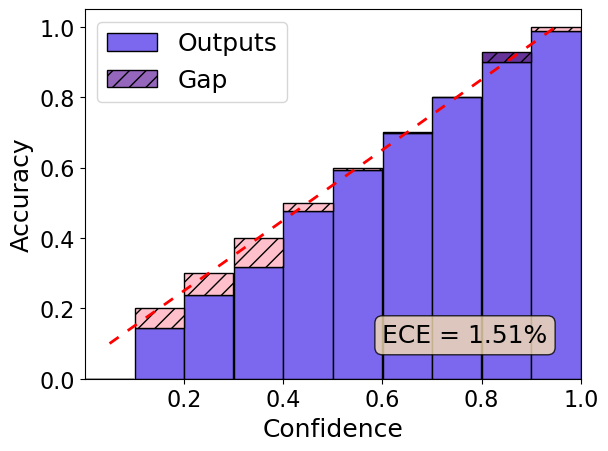}
		\caption{NED}
	\end{subfigure}
	\caption{Reliability diagrams for DDML model complemented with different approaches to estimate confidence of classification (CARS196 dataset). Our novel NED algorithm provides more accurate estimation of true correctness likelihood with the smallest expected calibration error (ECE).}
	\label{fig:reliability_diagrams}
\end{figure}

We analyze the impact of value of $k$ on the effectiveness of the proposed algorithm.
Figure~\ref{fig:k_preformance} shows the accuracy of the DDML model with different approaches based on the number of neighbors, $k$, used to detect the predicted label. Since the kNN method weights data points far from $z$ with the same importance as those close to $z$, its performance degrades with a larger $k$. Both variants of  WkNN are better than kNN because they consider the distances between $z$ and the neighbors $\{z_j\}_{j=1}^\numnn$. However, the weights they impose on samples are linear functions of the distances
whereas Theorem~\ref{theorem} shows that the correct weights should be exponential functions of the negative of the squared distance,
\ie,
$\exp(\|z - z_i\|_2^2/T)$,
hence should rapidly decrease as the distance increases.
This is why both kNN and WkNN show poor performance for large $k$s.

On the other hand, the performance of NED algorithm monotonically improves as $k$ grows.
This is due to Theorem~\ref{theorem},
\ie, (\ref{eq:theorem}) is the accurate value for the correctness likelihood
and (\ref{eq:ned-conf-level}) is a better approximation for (\ref{eq:theorem}) with a larger $k$.
Therefore, the performance of the NED algorithm should yield better results with larger $k$s.

Figure~\ref{fig:k_preformance} also shows a critical advantage of NED algorithm over the other methods.
The accuracy of NED algorithm does not vary much after a certain point.
Therefore NED is robust to the choice of $k$, which reduces the efforts of choosing $k$ considerably
whereas the other methods are sensitive to the choice of $k$ requiring much effort to find the optimal value for $k$.

This can also be explained by Theorem \ref{theorem}.
In (\ref{eq:ned-conf-level}), we can see  $\exp(-\|z - z_{i}\|_2^2)/T)$ become negligible for those points far from $z$ in the embedding space,
hence the additional accuracy gain obtained by increasing $k$ rapidly diminishes after a certain point.

\begin{figure}[h]
		\centering
	\begin{subfigure}{0.4\linewidth}
		\includegraphics[width=\linewidth]{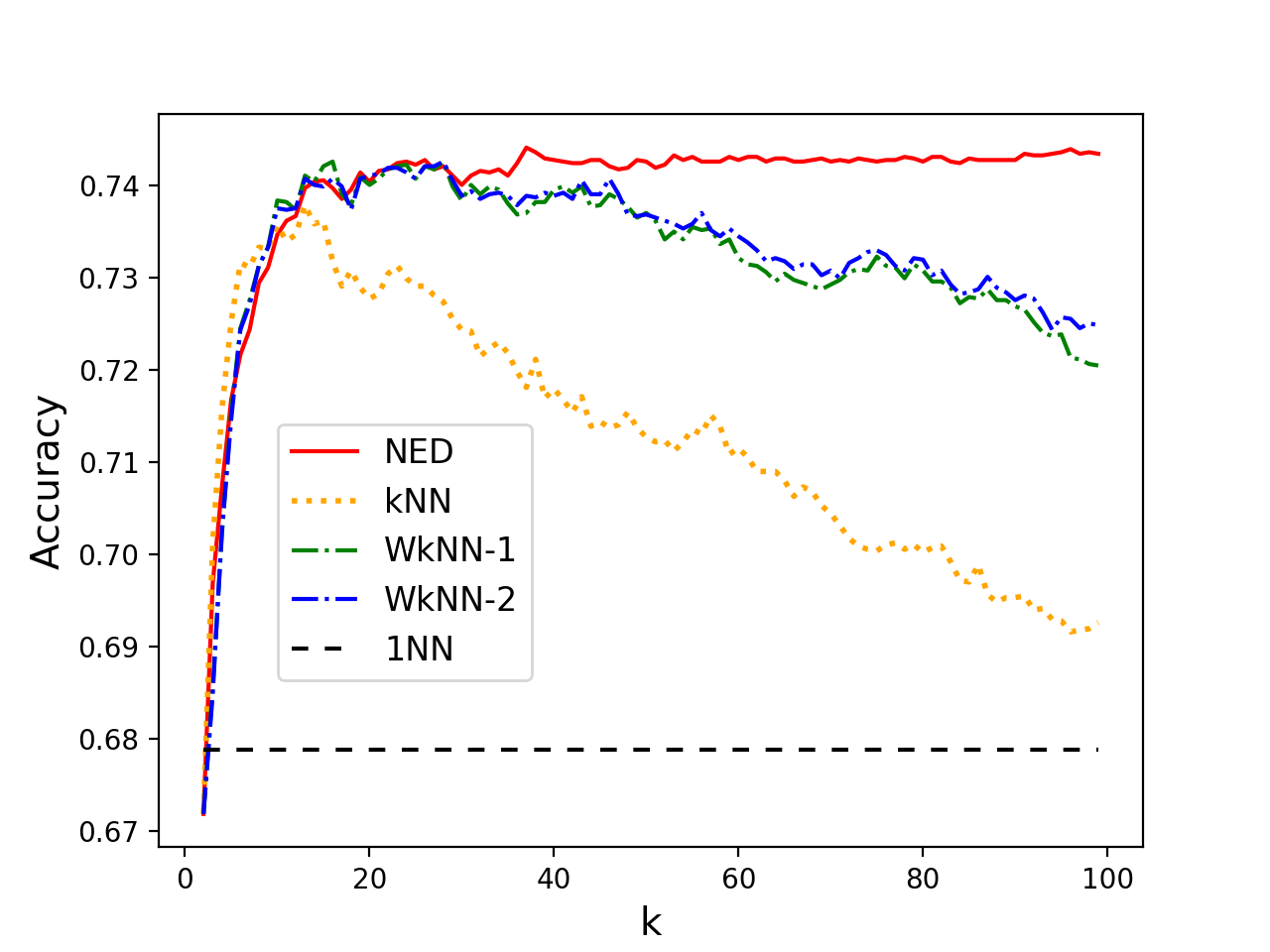}
		\caption{CUB}
	\end{subfigure}%
	\begin{subfigure}{0.4\linewidth}
		\includegraphics[width=\linewidth]{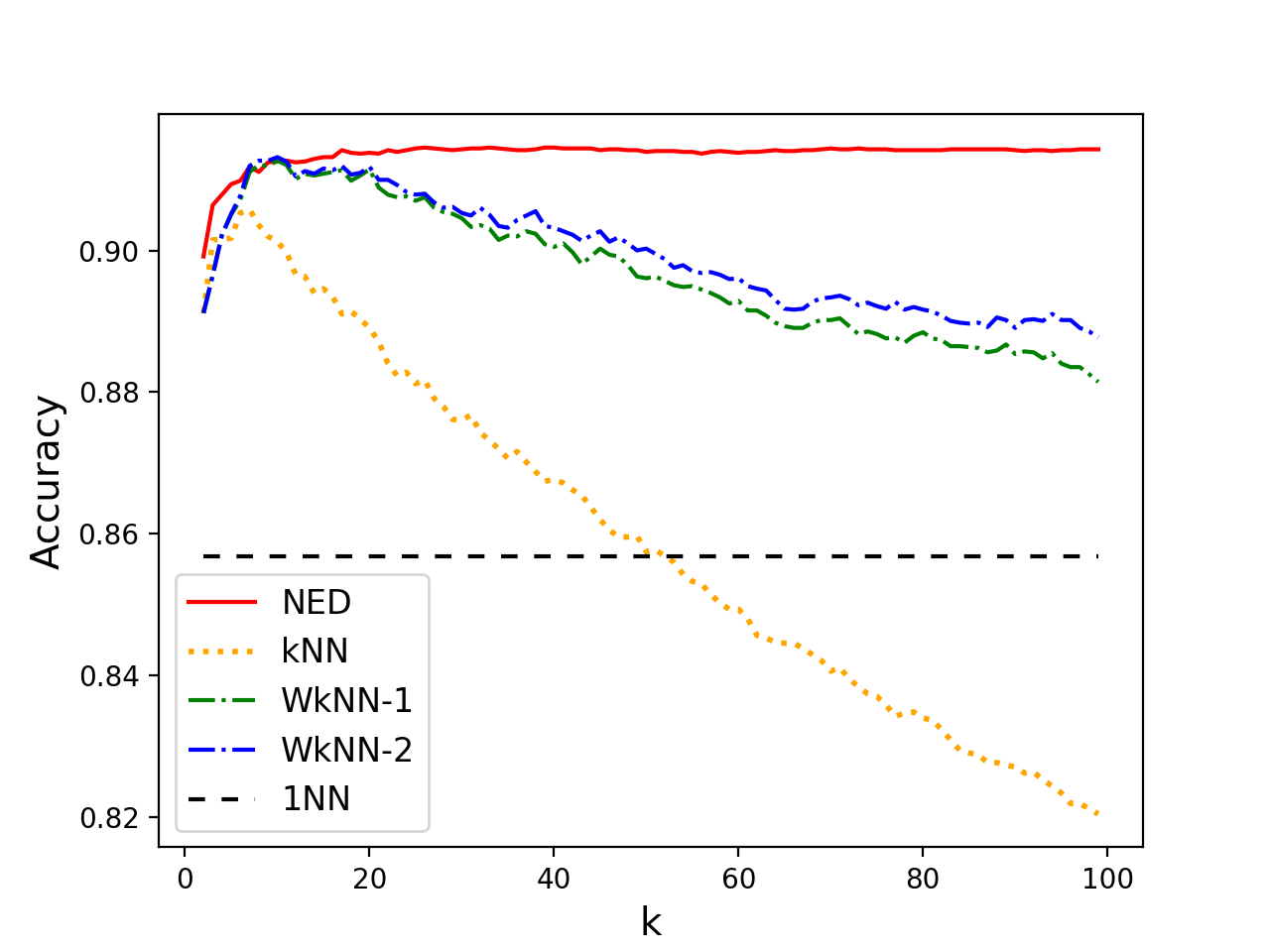}
		\caption{CARS}
	\end{subfigure}
	\begin{subfigure}{0.4\linewidth}
		\includegraphics[width=\linewidth]{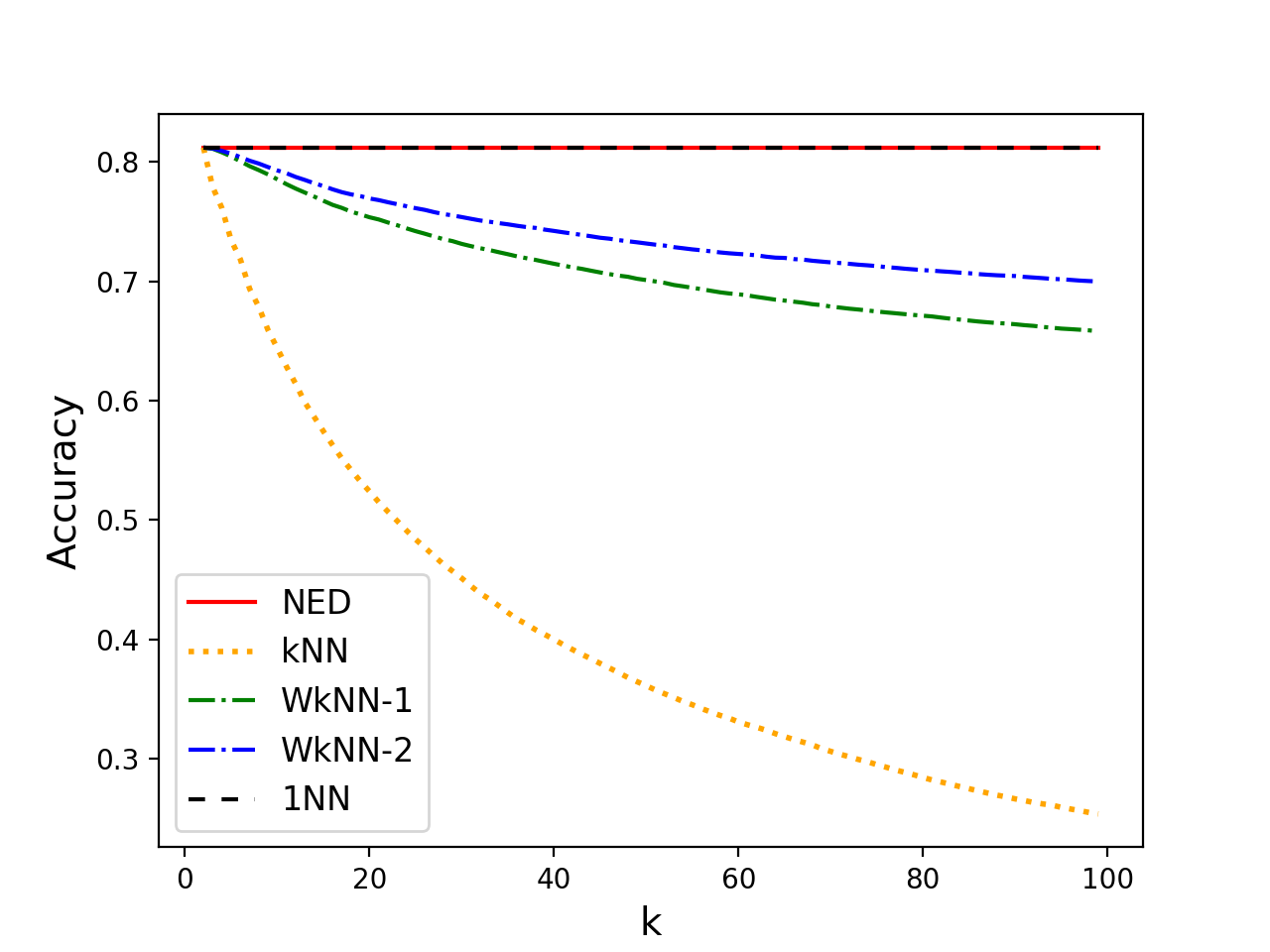}
		\caption{SOP}
	\end{subfigure}
	\begin{subfigure}{0.4\linewidth}
		\includegraphics[width=\linewidth]{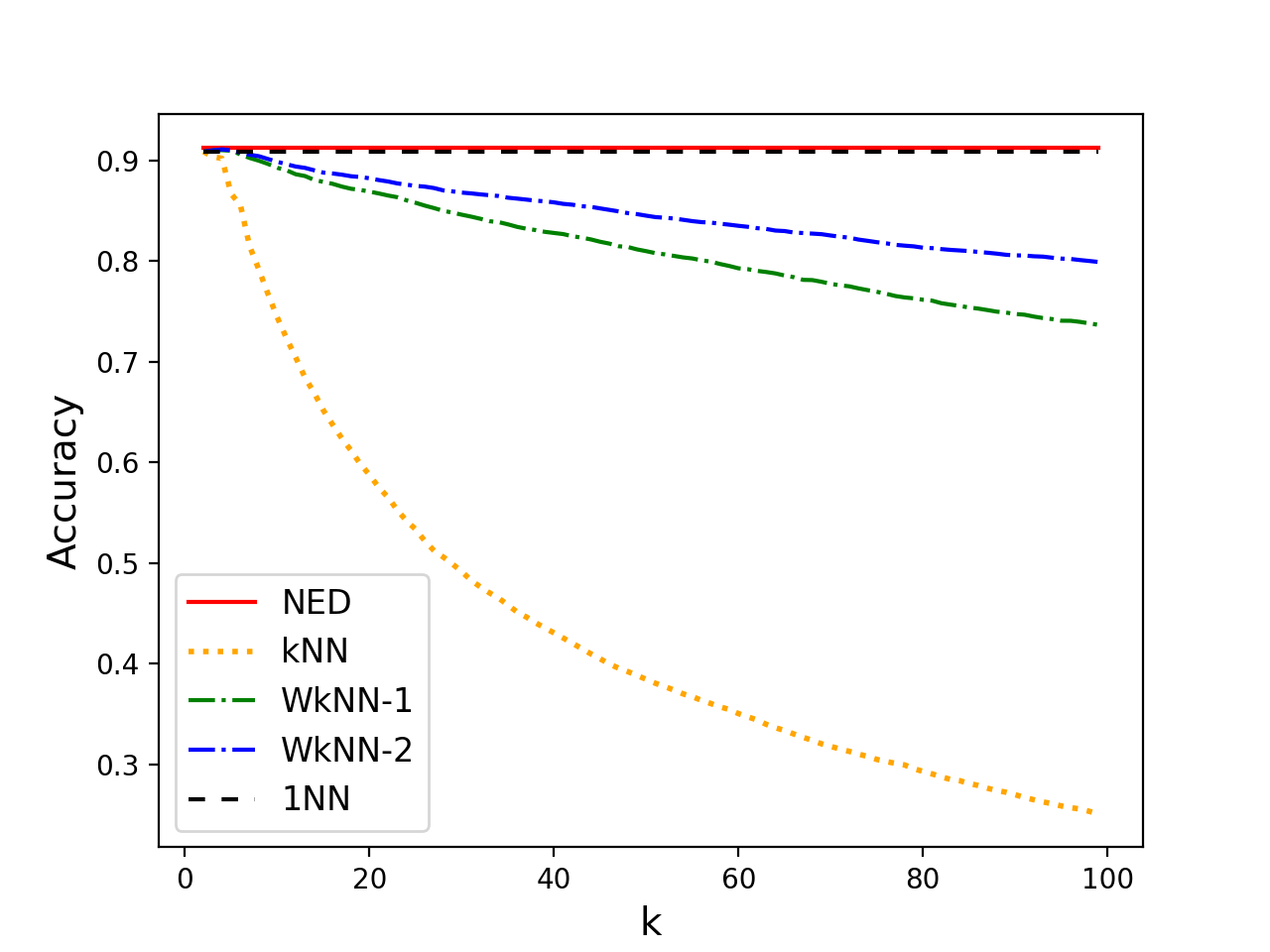}
		\caption{InShop}
	\end{subfigure}
	\caption{Accuracy of DDML model based on the $k$ value on CUB-200, CARS196,  SOP and InShop. }
	\label{fig:k_preformance}
\end{figure}

In Appendix~\ref{sec:ned-interpretation} 
we also provide the detailed interpretation of parameter tuning for $T$. Tuning $T$ corresponds to tuning the smoothing factor in the Gaussian kernel function. Hence we can interpret the tuning of $T$ as finding a better estimate for the true probability distribution function of the data.

\section{Distorted and Adversarial Images}
To evaluate the robustness of the proposed algorithm, we modified images in the test set with fifteen different types of corruption and five severity levels following the protocol described in~\cite{b36}.
Note that the support set was not altered and the value of $T$ is the same as that used for the former experiments without distortions.

Table~\ref{tab:distorted_preformance} shows the average accuracy and average ECE for DDML model completed with proposed NED algorithm and the four baselines (1NN, kNN, WkNN~\cite{b70}, and WkNN~\cite{b69}).
The results show that, on average, the model with the NED algorithm is more robust than the baselines when considering the types of image distortions.
In this case, the difference between the accuracy of model with NED algorithm and that of 1NN is even higher
compared to the results obtained on clean images,
\ie, the average accuracy for the model with NED algorithm is higher than that of 1NN by at least 1.6\% and at most 12\%.

\begin{table*}[h]
	\begin{center}
		\scalebox{0.98}{
		\begin{tabular}{|c|c|c|c|c|c|c|c|c|}
			\hline
			& \multicolumn{2}{|c|}{\textbf{CUB-200}} & \multicolumn{2}{|c|}{\textbf{CARS196}} & \multicolumn{2}{|c|}{\textbf{SOP}}  & \multicolumn{2}{|c|}{\textbf{InShop}} \\
			\hline
			& Accuracy & ECE & Accuracy & ECE  & Accuracy & ECE & Accuracy & ECE  \\
			\hline
			1NN & 52.8  & - & 72.9 & - & 74.1 & - & 80.2 & - \\
			\hline
			kNN & 60.5 & 12.3 & 76.0 & 14.1 & 75.1 & 27.3 & 81.8 & 34.4\\
			\hline
			WkNN \cite{b70} & 62.1 & 25.6  &77.0 & 16.8 & 75.2& 10.5 & 82.5 & 20.6  \\
			\hline
			WkNN \cite{b69} & 62.2 & 25.6 & 77.0 & 16.9 & 75.2& 10.5  & 82.5 &  20.6 \\
			\hline
			NED(proposed) & \textbf{64.8} &\textbf{3.0} & \textbf{77.2} & \textbf{2.3} &  \textbf{75.7}  & \textbf{2.9} & \textbf{85.6} & \textbf{0.9} \\
			\hline
		\end{tabular}
	}
		\caption{The average accuracy and average ECE for DDML model with proposed NED algorithm and the four baselines (1NN, kNN, WkNN~\cite{b70}, and WkNN~\cite{b69}) when images in test data are distorted with common corruption types~\cite{b36}.   The proposed approach consistently outperforms the baseline methods for all experiments.   }
		\label{tab:distorted_preformance}
	\end{center}
\end{table*}

To further examine the robustness of the models with the NED algorithm, we evaluate its performance on the adversarial examples. We craft adversarial samples using three popular white-box $L_\infty$ bounded untargeted attacks:
Fast Gradient Sign Method (FGSM)~\cite{b53},  Basic Iterative Method (BIM)~\cite{b57}, and Projected Gradient Descent (PGD)~\cite{b54}. We generate adversarial examples using PGD with 0.01 step size for 40 steps and using BIM with 0.02 step size for 10 steps during the training.

Table~\ref{tab:adversarial_preformance} shows the performance of DDML model with NED algorithm and the four baselines (1NN, kNN, WkNN~\cite{b70}, and WkNN~\cite{b69}) for CARS dataset. The proposed approach consistently outperforms baseline methods for all experiments.

\begin{table*}[h]
	\begin{center}
		\scalebox{0.95}{
		\begin{tabular}{|c|c|c|c|c|c|c|c|c|c|c|c|c|}
			\hline
			&\multicolumn{4}{|c|}{\textbf{\emph{FGSM}}} &\multicolumn{4}{|c|}{\textbf{\emph{BIM}}} &\multicolumn{4}{|c|}{\textbf{\emph{PGD}}} \\
			\hline
			& \multicolumn{2}{|c|}{$\epsilon = 0.1$}  &  \multicolumn{2}{|c|}{$\epsilon = 0.3$} &  \multicolumn{2}{|c|}{$\epsilon = 0.1$} &  \multicolumn{2}{|c|}{$\epsilon = 0.3$} &  \multicolumn{2}{|c|}{$\epsilon = 0.1$} &  \multicolumn{2}{|c|}{$\epsilon = 0.3$} \\
			\hline
			& A. & E. & A. & E.  & A. & E. & A. & E. & A. & E.  & A. & E.  \\
			\hline
			1NN & 86.3 & - & 77.7 & - & 85.7 & - & 69.5 & - &  86.4 & - & 74.4 & - \\
			\hline
			$k$NN & 87.5 & 10.4  & 79.6 & 9.8 & 87.3 & 10.8 & 71.8 & 9.6 & 87.8 & 10.6  & 76.1 & 10.9 \\
			\hline
			W$k$NN \cite{b70}& 88.6 & 12.7 & 80.9 & 12.7 & 88.4 & 13.2 & 73.0 & 13.1 & 88.9 & 12.8 & 77.7 & 14.5 \\
			\hline
			W$k$NN \cite{b69}& 88.6 & 12.9 & 80.9  & 12.9  & 88.4 & 13.3 & 73.2 & 13.2 & 88.9 & 12.9  & 77.7 & 14.5 \\
			\hline
			NED & \textbf{88.7} & \textbf{1.9} & \textbf{81.3} & \textbf{3.6} & \textbf{80.4} & \textbf{2.2}& \textbf{73.3} & \textbf{2.5} & \textbf{89.0} &  \textbf{2.0} & \textbf{78.0} & \textbf{2.9}  \\
			\hline
		\end{tabular}
	    }
		\caption{The accuracy (A.) and ECE (E.) for DDML model complemented with NED algorithm for adversarial images. The model with NED algorithm is more robust to adversarial images compared to the baselines.}
		\label{tab:adversarial_preformance}
	\end{center}
\end{table*}

\section{Conclusion}

In this paper we described a novel algorithm called NED that can be used with pre-trained DDML models
to improve the classification results while providing accurate approximation of true correctness likelihood.
We demonstrate the consistent performance improvement over the other baseline methods
for examples with normal data distribution,
those with distribution shifts related to common image corruptions,
and the adversarial examples.



{\small
	\bibliographystyle{plain}
	\bibliography{egbib}

\begin{thebibliography}{10}

\bibitem{normproxies}
Z.~Andrew and H.-Y. Wu.
\newblock Classification is a strong baseline for deep metric learning.
\newblock In {\em Proceedings of The British Machine Vision Conference}, 2019.

\bibitem{b64}
D.M. Bashtannyk and R.J. Hyndman.
\newblock Bandwidth selection for kernel conditional density estimation.
\newblock {\em Computational Statistics and Data Analysis}, 36:279--298, 2001.

\bibitem{b50}
A.~Ben-Tal, L.~E. Ghaoui, and A.~Nemirovski.
\newblock {\em Robust optimization}.
\newblock Princeton University Press, 2009.

\bibitem{b10}
K.~Bhatia, H.~Jain, P.~Kar, M.~Varma, and P.~Jain.
\newblock Sparse local embeddings for extreme multi-label classification.
\newblock In {\em Proceedings of the 29th Conference on Neural Information
  Processing Systems}, 2015.

\bibitem{contrastiveloss}
S.~Chopra, R.~Hadsell, and Y.~LeCun.
\newblock Learning a similarity metric discriminatively, with application to
  face verification.
\newblock In {\em Proceedings of the Conference on Computer Vision and Pattern
  Recognition}, 2005.

\bibitem{b72}
M.~H. DeGroot and S.~E. Fienberg.
\newblock The comparison and evaluation of forecasters.
\newblock {\em The statistician}, 1983.

\bibitem{b34}
S.~Dodge and L.~Karam.
\newblock A study and comparison of human and deep learning recognition
  performance under visual distortions.
\newblock In {\em Proceedings of the 26th international conference on computer
  communication and networks}, 2017.

\bibitem{b51}
S.~Dodge and L.~Karam.
\newblock A study and comparison of human and deep learning recognition
  performance under visual distortions.
\newblock In {\em Proceedings of 26th International Conference on Computer
  Communications and Networks}, 2017.

\bibitem{b70}
S.~A. Dudani.
\newblock The distance-weighted k-nearest-neighbor rule.
\newblock {\em IEEE Transactions on Systems, Man, and Cybernetics}, 1976.

\bibitem{b20}
N.~Frosst, N.~Papernot, and G.~Hinton.
\newblock Analyzing and improving representations with the soft nearest
  neighbor loss.
\newblock In {\em Proceedings of the 36th International Conference on Machine
  Learning}, 2019.

\bibitem{b42}
Y.~Gal.
\newblock {\em Uncertainty in Deep Learning}.
\newblock PhD thesis, University of Cambridge, 2016.

\bibitem{b41}
Y.~Gal and Z.~Ghahramani.
\newblock Dropout as a bayesian approximation: Representing model uncertainty
  in deep learning.
\newblock In {\em Proceeding of the 33rd International Conference on
  International Conference on Machine Learning}, 2016.

\bibitem{hierarchical_triplet}
W.~Ge, W.~Huang, D.~Dong, and M.~R. Scott.
\newblock Deep metric learning with hierarchical triplet loss.
\newblock In {\em Proceedings of the 14th European Conference on Computer
  Vision}, 2018.

\bibitem{b23}
J.~Gilmer, R.~P. Adams, I.~Goodfellow, D.~Anderson, and G.~E. Dahl.
\newblock Motivating the rules of the game for adversarial example research.
\newblock In {\em arXiv preprint arXiv:1807.06732}, 2018.

\bibitem{b53}
I.~J. Goodfellow, J.~Shlens, and C.~Szegedy.
\newblock Explaining and harnessing adversarial examples.
\newblock In {\em Proceedings of International Conference on Learning
  Representations}, 2015.

\bibitem{b69}
J.~Gou, L.~Du, Y.~Zhang, and T.~Xiong.
\newblock A new distance-weighted k-nearest neighbor classifier.
\newblock {\em Journal of Information and Computational Science}, 2012.

\bibitem{b68}
C.~Guo, G.~Pleiss, Y.~Sun, and K.~Q. Weinberger.
\newblock On calibration of modern neural netwroks.
\newblock In {\em Proceedings of the 34th International Conference on Machine
  Learning}, 2017.

\bibitem{b62}
K.~He, X.~Zhang, S.~Ren, and J.~Sun.
\newblock Identity mappings in deep residual networks.
\newblock In {\em Proceedings of the 14th European Conference on Computer
  Vision}, 2016.

\bibitem{b71}
K.~Hechenbichler and K.~Schliep.
\newblock Weighted k-nearest-neighbor techniques and ordinal classification.
\newblock {\em Discussion Paper No. 399, Collaborative Research Center 386.},
  2004.

\bibitem{b36}
D.~Hendrycks and T.~G. Dietterich.
\newblock Benchmarking neural network robustness to common corruptions and
  surface variations.
\newblock In {\em arXiv preprint arXiv:1807.01697}, 2019.

\bibitem{b37}
D.~Hendrycks, K.~Zhao, S.~Basart, J.~Steinhardt, and D.~Song.
\newblock Natural adversarial examples.
\newblock In {\em arXiv preprint arXiv:1907.07174}, 2019.

\bibitem{b5}
J.~R. Hershey, Z.~Chen, J.~Le~Roux, and S.~Watanabe.
\newblock Deep clustering: Discriminative embeddings for segmentation and
  separation.
\newblock In {\em Proceedings of the 41st IEEE International Conference on
  Acoustics, Speech and Signal Processing}, 2016.

\bibitem{tripletloss}
Elad Hoffer and Nir Ailon.
\newblock Deep metric learning using triplet network.
\newblock In {\em Proceedings of the Third International Workshop on Similarity
  Based Pattern Analysis and Recognition}, 2015.

\bibitem{b7}
H.~Hu, Y.~Wang, L.~Yang, P.~Komlev, L.~Huang, X.~Chen, J.~Huang, Y.~Wu,
  M.~Merchant, and A.~Sacheti.
\newblock Web-scale responsive visual search at bing.
\newblock In {\em Proceedings of the 24th ACM SIGKDD Conference on Knowledge
  Discovery and Data Mining}, 2018.

\bibitem{b28}
A.~Ilyas, S.~Santurkar, D.~Tsipras, L.~Engstrom, B.~Tran, and A.~Mandry.
\newblock Adversarial examples are not bugs, they are features.
\newblock In {\em Proceedings of the 33rd Conference on Neural Information
  Processing Systems}, 2019.

\bibitem{b61}
S.~Ioffe and C.~Szegedy.
\newblock Batch normalization: Accelerating deep network training by reducing
  internal covariate shift.
\newblock In {\em Proceedings of the International Conference on Machine
  Learning}, 2015.

\bibitem{attributes-based-confidence}
S.~Jha, S.~Raj, S.~Fernandes, S.~K. Jha, S.~Jha, B.~Jalaian, G.~Verma, and
  A.~Swami.
\newblock Attribution-based confidence metric for deep neural networks.
\newblock In {\em Proceedings of the 32nd International Conference on Neural
  Information Processing Systems}, 2019.

\bibitem{trust_or_not}
H.~Jiang, B.~Kim, M.~Guan, and M.~Gupta.
\newblock To trust or not to trust a classifier.
\newblock In {\em Proceedings of the 31st International Conference on Neural
  Information Processing Systems}, 2018.

\bibitem{WonsikAttentionEnsemble}
W.~Kim, B.~Goyal, K.~Chawla, J.~Lee, and K.~Kwon.
\newblock Attention-based ensemble for deep metric learning.
\newblock In {\em Proceedings of the 14th European Conference on Computer
  Vision}, 2018.

\bibitem{b12}
J.~Krause, M.~Stark, J.~Deng, and L.~Fei-Fei.
\newblock 3d object representations for fine-grained categorization.
\newblock In {\em 2013 IEEE International Conference on Computer Vision
  Workshops}, pages 554--561, Dec 2013.

\bibitem{b57}
A.~Kurakin, I.~J. Goodfellow, and S.~Bangio.
\newblock Adversarial examples in the physical world.
\newblock In {\em arXiv preprint arXiv:1607.02533}, 2017.

\bibitem{in-shop}
Z.~Liu, P.~Luo, S.~Qiu, X.~Wang, and X.~Tang.
\newblock Deepfashion: Powering robust clothes recognition and retrieval with
  rich annotations.
\newblock In {\em Proceedings of IEEE Conference on Computer Vision and Pattern
  Recognition}, 2016.

\bibitem{b54}
A.~Madry, A.~Makelov, L.~Schmidt, D.~Tsipras, and A.~Vladu.
\newblock Towards deep learning models resistant to adversarial attacks.
\newblock In {\em Proceedings of International Conference on Learning
  Representations}, 2018.

\bibitem{b56}
C.~Mao, Z.~Zhong, J.~Yang, C.~Vondrick, and B.~Ray.
\newblock Metric learning for adversarial robustness.
\newblock In {\em Proceedings of the 33rd International Conference on Neural
  Information Processing Systems}, 2019.

\bibitem{towards_nn_knows_when_fails}
A.~Meinke and M.~Hein.
\newblock Towards neural networks that provably know when they don't know.
\newblock In {\em Proceedings of Eighth International Conference on Learning
  Representatives}, 2020.

\bibitem{b1}
Y.~Movshovitz-Attias, A.~Toshev, T.~K. Leung, S.~Ioffe, and S.~Singh.
\newblock No fuss distance metric learning using proxies.
\newblock In {\em Proceedings of the IEEE International Conference on Computer
  Vision}, 2017.

\bibitem{confidence_bayesian_binning}
M.~P. Naeini, G.~Cooper, and M.~Hauskrecht.
\newblock Obtaining well calibrated probabilities using bayesian binning.
\newblock In {\em Proceedings of the Twenty-Ninth AAAI Conference on Artificial
  Intelligence}, 2015.

\bibitem{BIER}
M.~Opitz, G.~Waltner, H.~Possegger, and H.~Bischof.
\newblock Deep metric learning with bier: Boosting independent embeddings
  robustly.
\newblock {\em IEEE Transactions on Pattern Analysis and Machine Intelligence},
  2018.

\bibitem{b27}
N.~Papernot and P.~McDaniel.
\newblock Deep k-nearest neighbors: Towards confident, interpretable and robust
  deep learning.
\newblock In {\em arXiv preprint arXiv: 1803.04765}, 2018.

\bibitem{b25}
N.~Papernot, P.~McDaniel, X.~Wu, S.~Jha, and A.~Swami.
\newblock Distillation as a defense to adversarial perturbations against deep
  neural networks.
\newblock In {\em IEEE Symposium on Security and Privacy}, 2016.

\bibitem{b46}
N.~Papernot, P.~D. McDaniel, and I.~J. Goodfellow.
\newblock Transferability in machine learning: from phenomena to black-box
  attacks using adversarial samples.
\newblock In {\em arXiv preprint arXiv:1605.07277}, 2016.

\bibitem{pac}
S.~Park, O.~Bastani, N.~Matni, and I.~Lee.
\newblock Pac confidence sets for deep neural networks via calibrated
  prediction.
\newblock In {\em Proceedings of Eighth International Conference on Learning
  Representatives}, 2020.

\bibitem{platt_scalling}
J.~Platt.
\newblock Probabilistic outputs for support vector machines and comparisons to
  regularized likelihood methods.
\newblock {\em Advances in Large Margin Classifiers}, 1999.

\bibitem{b47}
L.~Schott, J.~Rauber, M.~Bethge, and W.~Brendel.
\newblock Towards the first adversarially robust neural network model on mnist.
\newblock In {\em Proceedings of the International Conference on Learning
  Representations}, 2019.

\bibitem{b16}
F.~Schroff, D.~Kalenichenko, and J.~Philbin.
\newblock Facenet: A unified embedding for face recognition and clustering.
\newblock In {\em Proceedings of the IEEE Conference on Computer Vision and
  Pattern Recognition}, 2015.

\bibitem{b2}
K.~Sohn.
\newblock Improved deep metric learning with multi-class n-pair loss objective.
\newblock In {\em Proceedings of the 30th International Conference on Neural
  Information Processing Systems}, 2016.

\bibitem{b3}
H.~O. Song, Y.~Xiang, S.~Jegelka, and S.~Savarese.
\newblock Deep metric learning via lifted structured feature embedding.
\newblock In {\em Proceedings of IEEE Conference on Computer Vision and Pattern
  Recognition}, 2016.

\bibitem{b26}
D.~Stutz, M.~Hein, and B.~Schiele.
\newblock Disentangling adversarial robustness and generalization.
\newblock In {\em Proceedings of the IEEE Conference on Computer Vision and
  Pattern Recognition}, 2019.

\bibitem{b60}
C.~Szegedy, V.~Vanhoucke, S.~Ioffe, J.~Shlens, and Z.~Wojna.
\newblock Rethinking the inception architecture for computer vision.
\newblock In {\em Proceedings of the IEEE conference on computer vision and
  pattern recognition}, 2016.

\bibitem{b11}
C.~Wah, S.~Branson, P.~Welinder, P.~Perona, and S.~Belongie.
\newblock The caltech-ucsd birds-200-2011 dataset.
\newblock Technical Report CNS-TR-2011-001, California Institute of Technology,
  2011.

\bibitem{b55}
E.~Wallace, S.~Feng, and J.~Boyd-Graber.
\newblock Interpreting neural networks with nearest neighbors.
\newblock In {\em Proceedings of the 2018 EMNLP Workshop BlackboxNLP}, 2018.

\bibitem{b58}
J.~Wang, K.-C. Wang, M.~Law, F.~Rudzicz, and M.~Brudho.
\newblock Centroid-based deep metric learning for speaker recognition.
\newblock In {\em Proceedings of the 44th International Conference on
  Acoustics, Speech, and Signal Processing}, 2019.

\bibitem{ddml_multi_similarity}
X.~Wang, X.~Han, W.~Huang, D.~Dong, and M.R. Scott.
\newblock Multi-similarity loss with general pair weighting for deep metric
  learning.
\newblock In {\em Proceedings of the Conference on Computer Vision and Pattern
  Recognition}, 2019.

\bibitem{b44}
Y.~Wang, S.~Jha, and K.~Chaudhuri.
\newblock Analysing the robustness of nearest neighbors to adversarial
  examples.
\newblock In {\em Proceedings of the International Conference on Machine
  Learning}, 2018.

\bibitem{ddml_deep_face_recognition}
Y.~Wen, K.~Zhang, Z.~Li, and Y.~Qiao.
\newblock A discriminative feature learning approach for deep face recognition.
\newblock In {\em Proceedings of the European Conference on Computer Vision},
  2016.

\bibitem{b4}
C.-Y. Wu, R.~Manmatha, A.~J. Smola, and P.~Kraehenbuehl.
\newblock Sampling matters in deep embedding learning.
\newblock In {\em Proceedings of IEEE Conference on Computer Vision and Pattern
  Recognition}, 2017.

\bibitem{ddml_scalable_nca}
Z.~Wu, A.A. Efros, and S.X. Yu.
\newblock Improving generalization via scalable neighborhood component
  analysis.
\newblock In {\em Proceedings of the European Conference on Computer Vision},
  2018.

\bibitem{distance-based-confidence}
C.~Xing, S.~Arik, Z.~Zhang, and T.~Pfister.
\newblock Distance-based learning from errors for confidence calibration.
\newblock In {\em Proceedings of Eighth International Conference on Learning
  Representatives}, 2020.

\bibitem{b59}
S.~Xiong, Y.~Zhang, D.~Ji, and Y.~Lou.
\newblock Distance metric learning for aspect phrase grouping.
\newblock In {\em Proceedings of the 26th International Conference on
  Computational Linguistics: Technical Papers}, 2016.

\bibitem{b9}
I.~E.-H. Yen, X.~Huang, P.~Ravikumar, K.~Zhong, and I.~Dhillon.
\newblock Pd-sparse: A primal and dual sparse approach to extreme multiclass
  and multilabel classification.
\newblock In {\em Proceedings of the 33rd International Conference on Machine
  Learning}, 2018.

\bibitem{confidence_histogram_binning}
B.~Zadrozny and C.~Elkan.
\newblock Obtaining calibrated probability estimates from decision trees and
  naive bayesian classifiers.
\newblock In {\em Proceedings of the International Conference on Machine
  Learning}, 2001.

\bibitem{confidence_isotonic_regression}
B.~Zadrozny and C.~Elkan.
\newblock Transforming classifier scores into accurate multiclass probability
  estimates.
\newblock In {\em Proceedings of the Eighth ACM SIGKDD International Conference
  on Knowledge Discovery and Data Mining}, 2002.

\bibitem{b73}
A.~Zhai and H.~Wu.
\newblock Classification is a strong baseline for deep metric learning.
\newblock In {\em Proceedings of British Machine Vision Conference}, 2019.

\bibitem{extreme_softmax}
A.~Zhai and H.~Wu.
\newblock Extreme classification via adversarial softmax approximation.
\newblock In {\em Proceedings of Eighth International Conference on Learning
  Representatives}, 2020.

\bibitem{b22}
W.~Zhang, J.~Yan, X.~Wang, and H.~Zha.
\newblock Deep extreme multi-label learning.
\newblock In {\em Proceedings of the 35th International Conference on Machine
  Learning}, 2018.

\bibitem{b8}
Y.~Zhang, P.~Pan, Y.~Zheng, K.~Zhao, Y.~Zhang, X.~Ren, and R.~Jin.
\newblock Visual search at alibaba.
\newblock In {\em Proceedings of the 24th ACM SIGKDD Conference on Knowledge
  Discovery and Data Mining}, 2018.

\bibitem{b6}
S.~Zheng, Y.~Song, T.~Leung, and I.~Goodfellow.
\newblock Improving the robustness of deep neural networks via stability
  training.
\newblock In {\em Proceedings of IEEE Conference on Computer Vision and Pattern
  Recognition}, 2016.

\end{thebibliography}
}

\pagebreak

\appendix
\section{The proof of Theorem~\ref{theorem}}
In this appendix, we show that with the NED algorithm we can provide a good approximation of the true correctness likelihood assuming that the empirical distribution function obtained using the Gaussian kernel smoothing function well estimates the true distribution of data in each class.
Therefore we can not only show by experiments that NED score outperforms other methods, such as experiments with 1NN, kNN, WkNN~\cite{b70}, and WkNN~\cite{b69}, but also provide the theoretical ground to support the experimental results.

\label{sec:ned-interpretation}
\begin{theorem}
	\label{theorem}
	The probability $p$ of $y$ belonging to $\class{j} \in \classset$ given embedding $z=f(x)$ and support set $\{(x_{i}, y_{i})\}_{i=1}^{N}$ can be approximate by:
	\begin{equation}
	\label{eq:theorem-1}
	\hat{p}(y\in c_{j}) \propto\
	\frac
	{\sum_{i=1}^{N} \exp\left( - \|z - z_{i}\|_2^2 /T \right) \ibelongk{y_i}{\class{j}}}
	{\sum_{i=1}^{N} \exp\left( - \|z - z_{i}\|_2^2 /T \right)}.
	\end{equation}
	where $T$ > 0 is a parameter to be tuned.
\end{theorem}

\begin{proof}
Suppose that there are \numclasses\ classes for multi-class classification problem
and the set of the classes is defined by $\classset = \{ \class{1}, \ldots, \class{\numclasses}\}$.
Let $X\in\reals^n$ and $Y\in\classset$ be
the random variables representing the independent variables and the associated class respectively
with the joint probability density function (PDF), $\prob_{X,Y}:\reals^n\times \classset \to \preals$.
Let $f:\reals^n\to\reals^\embdim$ represents the trained deep distance metric learning (DDML) model.
Then this defines the joint PDF of $Z=f(X)\in\reals^\embdim$ and $Y$, $\prob_{Z,Y}:\reals^\embdim\times\classset \to \preals$.
In general, we do not know these PDFs.
Here $n$ and $\embdim$ are the dimensions of the original and embedding spaces respectively.

The true confidence score of a data point $x\in\reals^n$ for \ordinal{i}\ class
is
the probability of $y$ belongs to \class{j}\ given $z=f(x)\in\reals^\embdim$ in the embedding space,
\ie,
\begin{equation}
\label{eq:true-conf-score}
\Prob(y\in \class{j}|z=f(x))
= \frac{\prob_{Z|Y}(z|y \in \class{j}) \prob_Y(y \in \class{j})}{\prob_Z(z)}
\end{equation}
where $\prob_Y: \classset \to \preals$ and $\prob_Z:\reals^\embdim \to\preals$
are the marginal PDF of $Y$ and $Z$ respectively
and $\prob_{Z|Y}:\reals^\embdim \times \classset \to \preals$
is the conditional PDF of $Z$ given $Y$.
Since we do not know the joint PDF $\prob_{Z,Y}$,
we cannot calculate any of these three quantities.
Here we estimate these quantities using the data we are given.

Let us assume that we are given $N$ data points,
\ie, $\{(x_{i}, y_{i})\}_{i=1}^N$ with $x_i \in \reals^n$ and $y_i \in \classset$
where these are the samples from $\prob_{X,Y}$
and that $N_j$ is the number of data points belonging to the \ordinal{j}\ class,
\ie,
$N_j = \sum_{i=1}^N \ibelongk{y_i}{\class{j}}$ for $j=1,\ldots,\numclasses$.
Since
\begin{equation}
\label{eq:conf-score-prop}
\Prob(y \in \class{j}|z)
\propto
\prob_{Z|Y}(z|y \in \class{j})
\prob_Y(y \in\class{j})
\end{equation}
we need to estimate only two quantities $\prob_Y(y \in\class{j})$ and $\prob_{Z|Y}(z|y \in \class{j})$.

For the first quantity, we can estimate it by counting the data points belonging to each class
assuming that the data set is balanced,
\ie, the size of data for each class is proportional to the true portion of each class.
Thus we have
\begin{equation}
\label{eq:approx-pi}
\probest_Y(y \in\class{j}) = N_j / N
\end{equation}
where $\probest_Y:\classset \to \preals$ denotes the estimate for the mass probability function (PMF) for $Y$,
\ie, $\prob_Y$.


For the second quantity,
we use a probability density estimation method.
If we knew the type of probability distribution the data of each class in the embedding space come from,
we could use some parametric estimation method
to estimate the conditional probability density, $\prob_{Z|Y}(z|y \in \class{j})$.
However, we do not know the details of the distribution in general.
Therefore we use a non-parametric distribution estimation method
which does not assume a specific type of data distribution.
One such method uses the empirical probability density function (PDF),
\ie,
\begin{equation}
\label{eq:emp-pdf-delta}
\probest_{Z|Y}(z|y \in \class{j})
= \frac{1}{N_j} \sum_{i:y_j \in \class{j}} \delta(z-f(x_{i}))
= \frac{1}{N_j} \sum_{i:y_j \in \class{j}} \delta(z-z_{i})
\end{equation}
where $\delta:\reals^n\to\{0,\infty\}$ is a Dirac delta function
whose range is in the set of the extended real numbers
and $\probest_{Z|Y,\cdot}$ denotes an estimate for $\prob_{Z|Y}$.
This estimated PDF has infinite peaks, hence we cannot use it in practice.
However, we can apply the convolution operation with some smoothing kernel to (\ref{eq:emp-pdf-delta})
to obtain finite PDF estimate,
\ie,
\begin{equation}
\probest_{Z|Y,\cdot}(z|y \in \class{j}) = \probest_{Z|Y,\delta}(z|y \in \class{j}) \star g(z)
= \frac{1}{N_j} \sum_{i:y_j \in\class{j}} g(z-z_{i})
\end{equation}
where $g:\reals^\embdim\to\preals$ is a kernel smoothing function and $\star$ is the convolution operator.
One typical choice for the kernel smoothing function is the (multi-dimensional) Gaussian PDF.
In this case, the estimated PDF becomes
\begin{equation}
\label{eq:approx-each-pdf}
\probest_{Z|Y,f_{\Sigma_j}}(z|y\in\class{j}) = \frac{1}{N_j} \sum_{i:y_j \in \class{j}} f_{\Sigma_j}(z-z_i)
\end{equation}
where $f_{\Sigma_j}:\reals^\embdim\to\ppreals$
is the PDF of $\gauss(0,\Sigma_j)$ for some $\Sigma_j\in\posdefset{\embdim}$, \ie, zero mean Gaussian with $\Sigma_j$ as its covariance matrix.
We can interpret the convolution with the Gaussian kernel
as applying \embdim-dimensional low-pass filter.
Here $\posdefset{\embdim}$ denotes the set of all positive definite matrices in $\reals^{\embdim\times \embdim}$
and the PDF is defined by
\begin{equation}
\label{eq:gauss-kernel}
f_{\Sigma_j}
= \frac{1}{(2\pi)^{\embdim/2}\det(\Sigma_j)^{1/2}} \exp\left(-\frac{1}{2} z^T \Sigma_j^{-1} z\right).
\end{equation}
Note that the PDF in (\ref{eq:gauss-kernel}) is not used for representing a random variable,
but as a kernel smoothing function.

Now the right hand side of (\ref{eq:conf-score-prop}) can be approximated by
\begin{equation}
\label{eq:approx-cond-prob}
\probest_Y(y \in \class{j})
\probest_{Z|Y,f_{\Sigma_j}}(z|y \in \class{j})
= \frac{1}{N} \sum_{i:y_i \in \class{j}} f_{\Sigma_j}(z-z_{i})
\end{equation}
where (\ref{eq:approx-pi}) and (\ref{eq:approx-each-pdf}) are used.
Note that the approximation in (\ref{eq:approx-cond-prob}) does not depend on $N_j$ because it is cancelled out.
This is because the number of data in each class, $N_j$, accounts for the probability for each class.

Now (\ref{eq:conf-score-prop}) implies that the true confidence score can be approximated by
\begin{eqnarray}
\nonumber
\Prob(y \in \class{j}|z=f(x))
&\approx&
\frac
{\probest_Y(y \in \class{j}) \probest_{Z|Y}(z|y \in \class{j})}
{\sum_{j'=1}^\numclasses \probest_Y(y \in \class{j'}) \probest_{Z|Y}(z|y \in \class{j'})}
= \frac
{\sum_{i:y_i \in \class{j}} f_{\Sigma_j}(z-z_{i})}
{\sum_{j'=1}^\numclasses\sum_{i:y_i\in\class{j'}} f_{\Sigma_{j'}}(z-z_{i})}
\\
&=&
\frac
{\sum_{i:y_i \in \class{j}} \exp\left( - (z-z_{i})^T \Sigma_j^{-1} (z-z_{i})/2 \right)/{\det(\Sigma_j)^{1/2}}}
{\sum_{j'=1}^\numclasses \sum_{i:y_i \in \class{j'}} \exp\left( - (z-z_{i})^T \Sigma_{j'}^{-1} (z-z_{i})/2 \right)/{\det(\Sigma_{j'})^{1/2}}}
\label{eq:approx-conf-score}
\end{eqnarray}
where we use the fact that $\sum_{j=1}^\numclasses \Prob(y \in \class{j}\mid z=f(x)) = 1$
for the derivation of (\ref{eq:approx-conf-score}) from (\ref{eq:approx-cond-prob}).


In our work, we use DDML models that assume during the training that all the coordinates of the embedding variable $Z$ are independent and properly normalized for each class. Therefore, we can replace  $\Sigma_j$ with $\alpha_j I_\embdim$ where $\alpha_j >0$ and $I_\embdim\in\posdefset{\embdim}$ denotes the identity matrix.
Then (\ref{eq:approx-conf-score}) becomes
\begin{equation}
\label{eq:approx-conf-score-identity-cov}
\Prob(y \in \class{j}|z=f(x))
\approx
\frac
{\sum_{i:y_i \in \class{j}} \exp\left( - \|z-z_{i}\|_2^2 /2\alpha_j \right)/{\alpha_j ^{\embdim/2}}}
{\sum_{j'=1}^\numclasses \sum_{i:y_i \in \class{j'}} \exp\left( - \|z-z_{i}\|_2^2 /2\alpha_{j'} \right)/{\alpha_{j'} ^{\embdim/2}}}
\end{equation}
where $\|\cdot\|_2$ denotes the $\ell_2$ norm.
If we further assume that the variance of the data in the embedding space is equal for all classes,
\ie, $\alpha_1 = \cdots = \alpha_\numclasses = \alpha > 0$,
we have
\begin{eqnarray}
\nonumber
\lefteqn{
\Prob(y \in \class{j}|z=f(x))
\approx
\frac
{\sum_{i:y_i \in \class{j}} \exp\left( - \|z-z_{i}\|_2^2 /2\alpha \right)}
{\sum_{j'=1}^\numclasses \sum_{i:y_i=\class{j'}} \exp\left( - \|z-z_{i}\|_2^2 /2\alpha \right)}
}
\\
&=&
\frac
{\sum_{i:y_i \in \class{j}} \exp\left( - \|z-z_{i}\|_2^2 /2\alpha \right)}
{\sum_{i=1}^N \exp\left( - \|z-z_{i}\|_2^2 /2\alpha \right)}
=
\frac
{\sum_{i=1}^N \exp\left( - \|z-z_{i}\|_2^2 /2\alpha \right) \ibelongk{y_i}{\class{j}}}
{\sum_{i=1}^N \exp\left( - \|z-z_{i}\|_2^2 /2\alpha \right)}.
\label{eq:approx-conf-score-same-cov}
\end{eqnarray}

If we replace  $\alpha$ with $T/2$, we obtain the equation~(\ref{eq:theorem-1}).
\end{proof}

\section{Experiments with different DDML approaches}
\label{sec:results-ddmls}
We evaluate the performance of the proposed NED algorithm on the following three additional DDML
models:

\textbf{Triplet Semihard.} We train a model with triplet loss and semi-hard sample
mining~\cite{b16}. We use GoogleNet~\cite{b60} with batch normalization
~\cite{b61} replacing the final dense layer to match the
embedding dimension before triplet loss. Due to batch-size
constraints, we omitted the experiments with this model on
the SOP dataset.

\textbf{N Pairs.}  Following the implementation details described in~\cite{b2},
we train a model with N pair loss, which is defined as a
softmax cross-entropy loss function on the pairwise distances
within each batch. The approach with N pair loss is known
to perform better than the triplet variant. Batch composition
strategy samples pairs of images from N unique classes. We
add one dense layer and normalization using L2 norm to
GoogleNet with batch normalization.

\textbf{Proxy NCA.} We train a model with proxy NCA loss~\cite{b1}. We replace
positive and negative samples with points that represent the
ideal cluster center of each class - proxies, which are initialized
randomly and learned along with the embedding function.
The model is not sensitive to the batch selection process
and converges faster. We add one dense layer and normalization
using the L2 norm to a pretrained Resnet50~\cite{b62}. Since
both proxies and embeddings are normalized and training
can stall when relative distances become very small, we add
a temperature parameter equal to 0.33 to the NCA loss.

The results presented in the Table~\ref{tab:best_retreival_preformance} demonstrates the potential of using NED algorithm with different DDML models.

\begin{table*}[h]
	\begin{center}
		\begin{tabular}{|c|c|c|c|c|c|c|}
			\hline
			\multicolumn{7}{|c|}{\textbf{\emph{CUB-200}}} \\
			\hline
			& \multicolumn{2}{|c|}{\textbf{Triplet Semihard}} & \multicolumn{2}{|c|}{\textbf{N Pairs}} & \multicolumn{2}{|c|}{\textbf{Proxy NCA}}  \\
			\hline
			& Accuracy & ECE & Accuracy & ECE  & Accuracy & ECE    \\
			\hline
			1NN(reported) & - & - & \emph{50.96} & - & \emph{49.21} & -  \\
			\hline
			1NN & 49.20 & - & 52.76 & - & 57.85 & -  \\
			\hline
			kNN & 55.93 & 4.40 & 59.53 & 10.21 & 63.47 & 8.64 \\
			\hline
			WkNN \cite{b70} & 56.51 & 13.65 & 59.81 & 12.50 & 63.51 &  18.20 \\
			\hline
			WkNN \cite{b69} & 56.51 & 13.65 & 59.93 & 13.01 & 63.51 &  18.20 \\
			\hline
			NED(proposed) & \textbf{56.78} &\textbf{3.62}  & \textbf{60.41} & \textbf{2.30} & \textbf{64.73} & \textbf{2.18}  \\
			\hline
			\multicolumn{7}{|c|}{\textbf{\emph{CARS196}}} \\
			\hline
			& \multicolumn{2}{|c|}{\textbf{Triplet Semihard}} & \multicolumn{2}{|c|}{\textbf{Npairs}} & \multicolumn{2}{|c|}{\textbf{Proxy NCA}} \\
			\hline
			& Accuracy & ECE & Accuracy & ECE  & Accuracy & ECE  \\
			\hline
			1NN(reported) & - & - & \emph{71.12} & - & \emph{73.22} & -  \\
			\hline
			1NN & 59.62 & 20.03  & 67.92 & 4.16 & 73.28 & - \\
			\hline
			kNN & 63.28 & 3.76 & 71.59 & 6.16  & 74.01  &  12.04  \\
			\hline
			WkNN \cite{b70}  & 64.26 & 8.95 & 72.60 & 16.78 & 75.68 &  17.04 \\
			\hline
			WkNN \cite{b69} & 64.26 & 8.95 & 73.01 & 16.00 & 75.68  & 17.04  \\
			\hline
			NED(proposed) & \textbf{64.32} & \textbf{3.73} & \textbf{73.15} & \textbf{2.57} & \textbf{76.29} & \textbf{2.33}   \\
			\hline
			\multicolumn{7}{|c|}{\textbf{\emph{SOP}}} \\
			\hline
			& \multicolumn{2}{|c|}{\textbf{Triplet Semihard}} & \multicolumn{2}{|c|}{\textbf{Npairs}} & \multicolumn{2}{|c|}{\textbf{Proxy NCA}}  \\
			\hline
			& Accuracy & ECE & Accuracy & ECE  & Accuracy & ECE   \\
			\hline
			1NN(reported) & - & - & \emph{67.73} & - & \emph{73.73} & -  \\
			\hline
			1NN & - & - & 68.38 & - &  74.09 & -  \\
			\hline
			kNN & - & - & 68.38 & 20.15 &74.09 &  23.24  \\
			\hline
			WkNN \cite{b70} & - & - &  69.45& 8.49 & 74.28  &  10.38  \\
			\hline
			WkNN \cite{b69} & - & - &   69.45 &9.01&74.28 &  10.39  \\
			\hline
			NED(proposed) & - & -  & \textbf{70.15} & \textbf{2.74} & 74.28 & \textbf{3.01}  \\
			\hline
		\end{tabular}
		\caption{The accuracy and ECE for DDML models trained with three different loss functions complemented with proposed NED algorithm and the four baselines (1NN, kNN, WkNN~\cite{b70}, and WkNN~\cite{b69}) on CUB-200, CARS196 and SOP datasets.  The best results of each column are in \textbf{bold}. The proposed approach consistently outperforms baseline methods for all experiments. }
		\label{tab:best_retreival_preformance}
	\end{center}
\end{table*}

\end{document}